\def\year{2019}\relax
\documentclass[letterpaper]{article} 
\usepackage{aaai19}  
\usepackage{times}  
\usepackage{helvet}  
\usepackage{courier}  
\usepackage{url}  
\usepackage{graphicx}  

\usepackage{amsmath}
\usepackage{diagbox}
\usepackage{color}
\usepackage{amsmath,amssymb,amsfonts,dsfont}
\usepackage{multirow}
\renewcommand{\a}{\alpha}
\renewcommand{\b}{\beta}
\newcommand{\w}{\omega}
\newcommand{\eps}{\varepsilon}
\newcommand{\CALD}{\mathcal{D}}

\newcommand{\BBR}{\mathbb{R}}
\newcommand{\rpl}[2]{\textcolor{red}{#1}\textcolor{blue}{[#2]}}

\newcommand{\norm}[1]{\left\Vert#1\right\Vert}
\newcommand{\set}[2]{\{#1\colon#2\}}

\newtheorem{proposition}{Proposition}

\newenvironment{proof}[1][Proof]{\par\noindent{\bf #1.~}}
{\hspace*{\fill}\rule{5pt}{5pt}\par\vspace{0.2em}}
 
\newcommand{\commentout}[1]{{}}
\newcommand{\npcite}[1]{\citeauthor{#1} \citeyear{#1}}
\newcommand{\emcite}[1]{\citeauthor{#1} (\citeyear{#1})}
\begin{document}
%
\title{How Compact?: Assessing Compactness of Representations \\ through Layer-Wise Pruning}
\author{
Hyun-Joo Jung$^1$\thanks{Equally contributed.}
Jaedeok Kim$^1$\footnotemark[1]
Yoonsuck Choe$^{1,2}$\\
$^1$Machine Learning Lab, Artificial Intelligence Center, Samsung Research, Samsung Electronics Co.\\
56 Seongchon-gil, Secho-gu,
Seoul, Korea, 06765 \\
$^{2}$Department of Computer Science and Engineering, Texas A\&M University\\
College Station, TX, 77843, USA
}

\maketitle
\begin{abstract}
Various forms of representations may arise in the many layers embedded in deep neural networks (DNNs). Of these, where can we find the most compact representation?
We propose to use a pruning framework to answer this question: How compact can each layer be compressed, without losing performance?
Most of the existing DNN compression methods do not consider the relative compressibility of the individual layers.
They uniformly apply a single target sparsity to all layers or adapt layer sparsity using heuristics and additional training.
We propose a principled method that automatically determines the sparsity of individual layers derived from the importance of each layer.
To do this, we consider a metric to measure the importance of each layer based on the layer-wise capacity.
Given the trained model and the total target sparsity, we first evaluate the importance of each layer from the model.
From the evaluated importance, we compute the layer-wise sparsity of each layer.
The proposed method can be applied to any DNN architecture and can be combined with any pruning method that takes the total target sparsity as a parameter.
To validate the proposed method, we carried out an image classification task with two types of DNN architectures on two benchmark datasets and used three pruning methods for compression.
In case of VGG-16 model with weight pruning on the ImageNet dataset, we achieved up to 75\% (17.5\% on average) better top-5 accuracy than the baseline under the same total target sparsity.
Furthermore, we analyzed where the maximum compression can occur in the network. This kind of analysis can help us identify the most compact representation within a deep neural network.
\end{abstract}

\section{INTRODUCTION} 
\noindent In recent years, DNN models have been used for a variety of artificial intelligence (AI) tasks such as image classification \cite{simonyan2014very,he2016deep,huang2017densely}, semantic segmentation \cite{he2017mask}, and object detection \cite{redmon2017yolo9000}.
The need for integrating such models into devices with limited on-board computing power have been growing consistently.
However, to extend the usage of large and accurate DNN models to resource-constrained devices such as mobile phones, home appliances, or IoT devices, compressing DNN models while maintaining their performance is an imperative.

A recent study \cite{zhu2017prune} demonstrates that making large DNN models sparse by pruning can consistently outperform directly trained small-dense DNN models. 
In pruning DNN models, we would like to address the following problem: \textit{``How can we determine the target sparsity of individual layers?"}.
Most existing methods set the layer-wise sparsity to be uniformly fixed to a single target sparsity \cite{zhu2017prune} or adopt layer-wise sparsity manually \cite{han2015learning,he2017channel}. 

Starting from the assumption that not all layers in a DNN model have equal importance, we propose a new method that automatically computes the sparsity of individual layers according to layer-wise importance.

The contributions of the proposed method are as follows:

\begin{itemize}
\item The proposed method can analytically compute the layer-wise sparsity of DNN models. The computed layer-wise sparsity value enables us to prune the DNN model more efficiently than pruning with uniform sparsity for all layers. 
In our experiments, we validate that the proposed layer-wise sparsity scheme can compress DNN models more than the uniform-layer sparsity scheme while retaining the same classification accuracy.
\item The proposed method can be combined with any pruning method that takes total target sparsity as an input parameter. 
Such a condition is general in compression tasks because users usually want to control the trade-off between compression ratio and performance. 
In our experiments, we utilized three different pruning approaches (weight pruning by \npcite{han2015learning}, random channel pruning, and channel pruning by \npcite{li2016pruning}) for compression. 
Other pruning approaches are also applicable.
\item The proposed method can be applied to any DNN architecture because it does not require any constraint on the DNN architectures (e.g., \npcite{liu2017learning}; \npcite{ye2018rethinking} requires batch normalization layer to prune DNN model). 
\item To compute the layer-wise sparsity, we do not require additional training or evaluation steps which take a long time. 
\end{itemize}

\section{RELATED WORKS}
Pruning is a simple but efficient method for DNN model compression.
\emcite{zhu2017prune} showed that making large DNN models sparse by pruning can outperform small-dense DNN models trained from scratch.
There are many pruning methods according to the granularity of pruning from weight pruning \cite{han2015learning,han2015deep} to channel pruning \cite{he2017channel,liu2017learning,li2016pruning,molchanov2016pruning,ye2018rethinking}. 
These approaches mainly focused on how to select the redundant weights/filters in the model, rather than considering how many weights/filters need to be pruned in each layer. 

However, considering the role of each layer is critical for efficient compression.
\emcite{arora2018stronger} measured the noise sensitivity of each layer and tried to use it to compute the most effective number of parameters for each layer.

Recently, \emcite{he2018adc} proposed Automated Deep Compression (ADC).
They aimed to automatically find the sparsity ratio of each layer which is similar to our goal.
However, they used reinforcement learning to find the sparsity ratio and characterized the state space as a layout of layer, i.e., kernel dimension, input size, and FLOPs of a layer, etc.

In this paper, we focus on directly measuring the importance of each layer for the given task and the model. 
We used the value of the weight matrix itself rather than the layout of the layer.
We will show considering layer-wise importance in DNN compression is effective.

\section{PROBLEM FORMULATION}
Our goal is to compute the sparsity of each layer in the model given the total target sparsity of the model.
In this section we derive the relation between the total target sparsity $s$ of a model and the sparsity of each layer while considering the importance of each layer.

Let $s_l \in [0,1)$, $l=1,\cdots L$, be the layer sparsity of the $l$-th layer where $L$ is the total number of layers in the model.
Then, the layer sparsity $s_l$ should satisfy the following condition:
\begin{align}
\label{eqn:sparsity_condi}
	\sum_{l=1}^{L} {s_l}{N_l}=sN,
\end{align}
where $N_l$ is the number of parameters in the $l$-th layer and $N$ is the total number of parameters in the model.

A DNN model is usually overparameterized and it contains many redundant parameters.
Pruning aims at removing such parameters from every layer and leaving only the important ones.
So, if a large proportion of parameters in a layer is important for the given task, the layer should not be pruned too much while otherwise the layer should be pruned aggressively.
Hence, it is natural to consider the importance of a layer when we determine the layer sparsity.

We assume that the number of remaining parameters in the $l$-th layer after pruning is proportional to the importance of the $l$-th layer, $\omega_l$.
Because the number of remaining parameters for the $l$-th layer is equal to $(1-s_l)N_l$, we have,

\begin{align}
\label{eqn:remained_params}
(1-s_l)N_l = {\a}\omega_l,
\end{align}
where $\a > 0$ is a constant. 
Because we want all layers to equally share the number of pruned parameters with respect to the importance of each layer, $\a$ is set to be independent of the layer index. 

By summing both sides over all the layers, we obtain 
\begin{align}
	\label{eqn:sum_remained}
	\sum_{l=1}^{L} (1-s_l)N_l & = N-sN, \\ 
	\label{eqn:sum_importance}
	\sum_{l=1}^{L} {\a}\omega_l & = {\a}\Omega,
\end{align}
where $\Omega$ is the sum of importance of all layers.
From (Eq. \ref{eqn:sum_remained}) and (Eq. \ref{eqn:sum_importance}), we can easily compute $\a$, i.e.,
\begin{align}
\label{eqn:compute_alpha}
{\a} = \frac{(1-s)N}{\Omega}.
\end{align}

The above equation can be seen as the ratio of the total number of remained parameters and the sum of importance of all layers which is proportional to the total number of effective parameters.

When $\a=1$, the number of remaining parameters determined by $s$ is the same as the number of effective parameters. 
In that case, we can readily prune parameters in each layer.
When $\a<1$, then the number of remaining parameters becomes smaller than the number of effective parameters.
Therefore, we need to prune some effective parameters.
In this case, $\a$ acts as a distributor that equally allocates the number of pruning effective parameters to all layers.
Thanks to the $\a$, the compression is prevented from pruning a specific layer too much. 
In the case of $\a>1$, it acts similar to the case of $\a<1$ but allocates the number of pruning redundant parameters instead.

Because all factors consisting $\a$ are automatically determined by the given model, we can control the above pruning cases by controlling $s$ only, which we desired.  

By combining (Eq. \ref{eqn:remained_params}) and (Eq. \ref{eqn:compute_alpha}), we obtain the layer-wise sparsity as
\begin{align}
\label{eqn:layer_wise_sparsity}
\nonumber s_l & =  1-{\a \frac{\w_l}{N_l}} \\
& =  1-(1-s)\frac{N}{N_l}\frac{\omega_l}{\Omega}.
\end{align}

The remaining problem is: \textit{How can we compute the layer-wise importance $\omega_l$}?
We propose metric to measure the importance of each layer based on the \textit{layer-wise capacity}.
In the subsection below, we will explain the metric in detail.

\subsection{Layer-wise Capacity}
We measure the importance of a layer using the layer capacity induced by noise sensitivity \cite{arora2018stronger}.
Based on the definition \cite{arora2018stronger}, the authors proved that a matrix having low noise sensitivity has large singular values (i.e., low rank).
They identified the effective number of parameters of a DNN model by measuring the capacity of mapping operations (e.g., convolution or multiplication with a matrix) which is inverse proportional to the noise sensitivity.
The layer capacity has an advantage of directly counting the effective number of parameters in a layer.
However, they only consider a linear network having fully connected layers or convolutional layers.

Motivated by the work, we use the concept of capacity to compute layer-wise sparsity.
The layer capacity $\mu_l$ of the $l$-th layer is defined as 
\begin{align}
\label{eqn:capacity}
	\mu_l := \max_{x^l\in S_l}\frac{||W^{l}x^{l}||}{||W^{l}||_{F} ||x^{l}||}, 
\end{align}
where $W^{l}$ is a mapping (e.g., convolution filter or multiplication with a weight matrix) of the $l$-th layer and $x^{l}$ is the input of the $l$-th layer.  
$||\cdot||$ and $||\cdot||_{F}$ are $l^2$ norm and the Frobenius norm of the operator, respectively.
$S_l$ is a set of inputs of the $l$-th layer.
In other words, $\mu_l$ is the largest number that satisfies $\mu_l ||W^{l}||_{F} ||x^{l}|| = ||W^{l}x^{l}||$. 

According to the work \cite{arora2018stronger} and (Eq. \ref{eqn:capacity}), the mapping having large capacity has low rank and hence small number of effective parameters.
Therefore, we let the effective number of parameters to be inverse proportional to the layer-wise capacity. 
More specifically, the effective number of parameters $e_l$ of the $l$-th layer can be written as,
\begin{align}
\label{eqn:effective_params}
e_l \propto \frac{\b}{\mu_l^2},
\end{align}
where $\b$ is a constant.

In fact, $\b$ might be different across layers because other attributes such as the depth of layers (distance from the input layer) would affect the number of effective parameters.
In this paper, however, we set $\b$ to be constant for simplicity and focus on the layer-wise capacity. 

We assume that the layer having a large number of effective parameters (in other words, having small capacity) should not be pruned too much.
Therefore, we can set the importance of the $l$-th layer $\omega_l$ to be the number of effective parameters $e_l$, i.e, $\omega_l = e_l$.
Then, (Eq.  \ref{eqn:layer_wise_sparsity}) becomes,
\begin{align}
\label{eqn:sparsity_capacity}
\nonumber s_l & = 1-(1-s)\frac{N}{N_l}\frac{e_l}{E}\\
& = 1-(1-s)\frac{N}{N_l}\frac{M}{\mu_{l}^{2}}, 
\end{align}
where $E$ and $M$ are the sum of $e_l$ and $1/{\mu_l^2}$ for all layers, respectively.
Because the value of $s_l$ is independent of the value of $\b$, $s_l$ can be obtained without any knowledge of the value of $\b$.

\section{PRUNING WITH LAYER-WISE SPARSITY}
In this section, we explain our proposed DNN model compression process with layer-wise sparsity.
Given the total target sparsity, we first compute the layer-wise importance of each layer from the two proposed metrics (layer-wise performance gain and layer-wise capacity).
According to (Eq. \ref{eqn:sparsity_capacity}), we can compute layer-wise sparsity easily.

However, there are some issues in computing the actual layer-wise sparsity.
First, the user might want to control the minimum number of remaining parameters.
In the worst case, the required number of pruned parameters for a layer might be equal or larger than the total number of parameters in the layer. 
Second, the exact number of pruned parameters would be different from the computed sparsity from (Eq. \ref{eqn:layer_wise_sparsity}).
For example, in channel pruning, the number of pruned parameters in the $(l-1)$th layer also affects the $l$-th layer. 

To handle those problems, we re-formulate (Eq. \ref{eqn:layer_wise_sparsity}) as an optimization problem.
The objective function is defined as, 
\begin{align}
\label{eqn:optimization}
	\min_{\epsilon}  & \quad ||\epsilon||^2, \\
\nonumber 
	\textrm{such that }  
		& \quad \xi_l \leq \a (1+\varepsilon_l) \omega_l \leq N_l \textrm{ for all $l = 1, \cdots, L$},\\
\nonumber 
		& \quad \sum_{l} \a (1+\varepsilon_l) \omega_l = (1-s)N,
\end{align}
where $\xi_l$ is the minimum number of remaining parameters after pruning.  
$\epsilon = (\varepsilon_1, \cdots, \varepsilon_L)$ is the vector of layer-wise perturbations on the constant $\a$.
That is, though we set the number of pruned parameters of each layer equally proportional to the importance of the layer, we perturb the degree of pruning of each layer by adding $\varepsilon_l$ to $\a$ in unavoidable cases.

\begin{proposition}
	\label{prop:sol}
	If $\sum_{l=1}^L \xi_l \leq (1-s)N$, 
	then a solution of the optimization problem (\ref{eqn:optimization}) exists and it is unique.
\end{proposition}
\begin{proof}
	Since $\norm{\epsilon}^2$ is strictly convex in $\epsilon$, the optimization problem (\ref{eqn:optimization}) has a unique solution if the problem is feasible.
    So it is enough to show that the constraint set of the problem is not empty.
    Denote the constraint set of the inequalities by
    \begin{align*}
    	\CALD := \set{\epsilon \in \BBR^L}{\xi_l \leq \a (1+\eps_l) \w_l \leq N_l \textrm{ for $l=1,\cdots,L$}}.
    \end{align*}
    Consider a continuous function $f\colon \CALD \to \BBR$ such that $f(\epsilon) := \sum_{l=1}^L \a (1 + \eps_l) w_l$.
    If we take $\epsilon_{max} = (N_1 / \a w_1 - 1,\cdots, N_L / \a w_L - 1)$, 
    then $\epsilon_{max} \in \CALD$ and
    \begin{align*}
    	f(\epsilon_{max}) = \sum_{l=1}^L N_l = N \geq (1-s)N.
    \end{align*}
    Similarly, by taking $\epsilon_{min} = (\xi_1 / \a w_1 - 1,\cdots, \xi_L / \a w_L - 1)$, we have $\epsilon_{min} \in \CALD$ and
    \begin{align*}
    	f(\epsilon_{min}) = \sum_{l=1}^L \xi_l \leq (1-s)N.
    \end{align*}
    By the provided condition, we know $\sum_{l=1}^L \xi_l \leq (1-s)N \leq N$.
    Then the continuity of $f$ yields that there exists a point $\tilde{\epsilon}\in\CALD$ such that $f(\tilde{\epsilon}) = (1 - s) N$, which completes our proof.
\end{proof}

Intuitively, Proposition \ref{prop:sol} says that the total target sparsity $s$ of pruning should not violate the constraint of the minimum number of remaining parameters.
Let us assume $\sum_{l=1}^L \xi_l \leq (1-s)N$ to guarantee the feasibility of the optimization problem.
We then can apply vairous convex optimization algorithms to obtain the optimal solution $\epsilon^*$  \cite{boyd2004convex}.

Remark that our proposed method does not require model training or evaluation to find the values of layer sparsities.
We obtain analytically the layer sparsity without additional training or evaluation.
Many existing works \cite{zhong2018where,he2018adc} use a trial and error approach for searching for the best combination of hyperparameters such as the layer sparsity $s_l$. 
Evaluating each combination through additional training or evaluation is necessary which makes these approaches not feasible on large datasets.

\begin{table}
  \caption{Simple DNN model architecture. Conv and FC mean convolutional and fully connected layer, respectively.}
  \label{tbl:keras-model}
  \centering
  \begin{tabular}{clclclcl}
    \hline
    Layer & Filter size & Output size & Activation\\
    \hline
    Conv1       & $(3,3,3,32)$  & $(32,32,32)$ & ReLU        \\
    Conv2       & $(3,3,32,32)$ & $(32,32,32)$ & ReLU        \\
    Maxpool   & $(2,2)$       & $(16,16,32)$ &        \\
    Conv3       & $(3,3,32,64)$ & $(16,16,64)$ & ReLU        \\
    Conv4       & $(3,3,64,64)$ & $(16,16,64)$ & ReLU        \\
    Maxpool   & $(2,2)$       & $(8,8,32)$   &        \\
    FC1         & $(2048,512)$  & $(1,512)$    & ReLU    \\
    FC2         & $(512,10)$    & $(1,10)$     & Softmax   \\
    \hline    
  \end{tabular}
\end{table}

Our method requires convex optimization in a small dimensional space that has negligible computing time in general.
The total calculation time mainly depends on the time spent by calculating the layer capacity $\mu_l$ by (Eq.\ \ref{eqn:capacity}).
Although it is required to compute the norm of each input vectors induced by the whole dataset, 
it has approximately similar computing costs as inferencing.
Our approach has an advantage in terms of computation time.

With our implementation computing (Eq. \ref{eqn:capacity}) took less than 3 hours\commentout{\it \rpl{[*yschoe* for what???]}{[*yschoe* this is still unclear --- *jdkim* solving (\ref{eqn:optimization}) including the computation of the capacities of the whole layers]}} for VGG16 model with the ImageNet dataset in a single GPU machine.
For VGG-16 model with the CIFAR-10 dataset, the total computation time of (Eq. \ref{eqn:capacity}) was less than 1 minute on the same hardware setup.
The most time consuming part of our implementation was calculating the Frobenius norm of each layer.
There could be more speed up if we use an approximate value of norms or use a subset of the data set instead of the whole data set.
However, these are out of the scope of this paper so we will not discuss it any further.

Note that although here we assume that all layers in the model are compressed, selecting and compressing the subset of $1, \cdots, L$ can easily be handled.
Moreover, the process described below can also be applied to any pruning methods such as channel pruning or weight pruning.

\subsection{Selection of Target Sparsity for Channel Pruning}
In the experiments, we applied the proposed layer-wise sparsity method to both weight pruning and channel pruning. 

However, in case of channel pruning, a pruned layer affects the input dimension of the next layer.
So the actual number of pruned parameters may differ from our expectation. 
As a consequence, our proposed method does not achieve the exact total target sparsity $s$  (over-pruned in usual) by channel pruning.
In fact, the exact number of remaining parameters by channel pruning depends on the topology of a neural network and properties of each layer, e.g. the kernel size of a convolution layer.
So an exact formulation of the number of remaining parameters is not mathematically tractable.

To overcome such a limitation we need to consider a method to achieve the total target sparsity. 
Because the value of $s$ in our proposed method can be considered as the compression strength in pruning, we can find the exact sparsity $\hat{s}$ from the $s$.
Let $C(s)$ be the actual number of remaining parameters in a model after applying channel pruning with sparsity $s$, i.e., the $l$-th layer is pruned by the sparsity $s_l$ induced by (\ref{eqn:optimization}).
Then the total target sparsity is achievable if we use the value $\hat{s}$ instead of the total target sparsity $s$.
Finally, the problem becomes finding proper $\hat{s}$ that satisfies,
\begin{align*}
	C(\hat{s}) = (1-s)N.
\end{align*}

Because both solving (\ref{eqn:optimization}) and counting the number of parameters require low computational costs,
the value of $\hat{s}$ can be obtained in reasonable time.
We therefore can control the achieved sparsity of the model by using $\hat{s}$ in case of channel pruning.

\begin{figure}
    \centerline{\includegraphics[width=0.85\columnwidth]{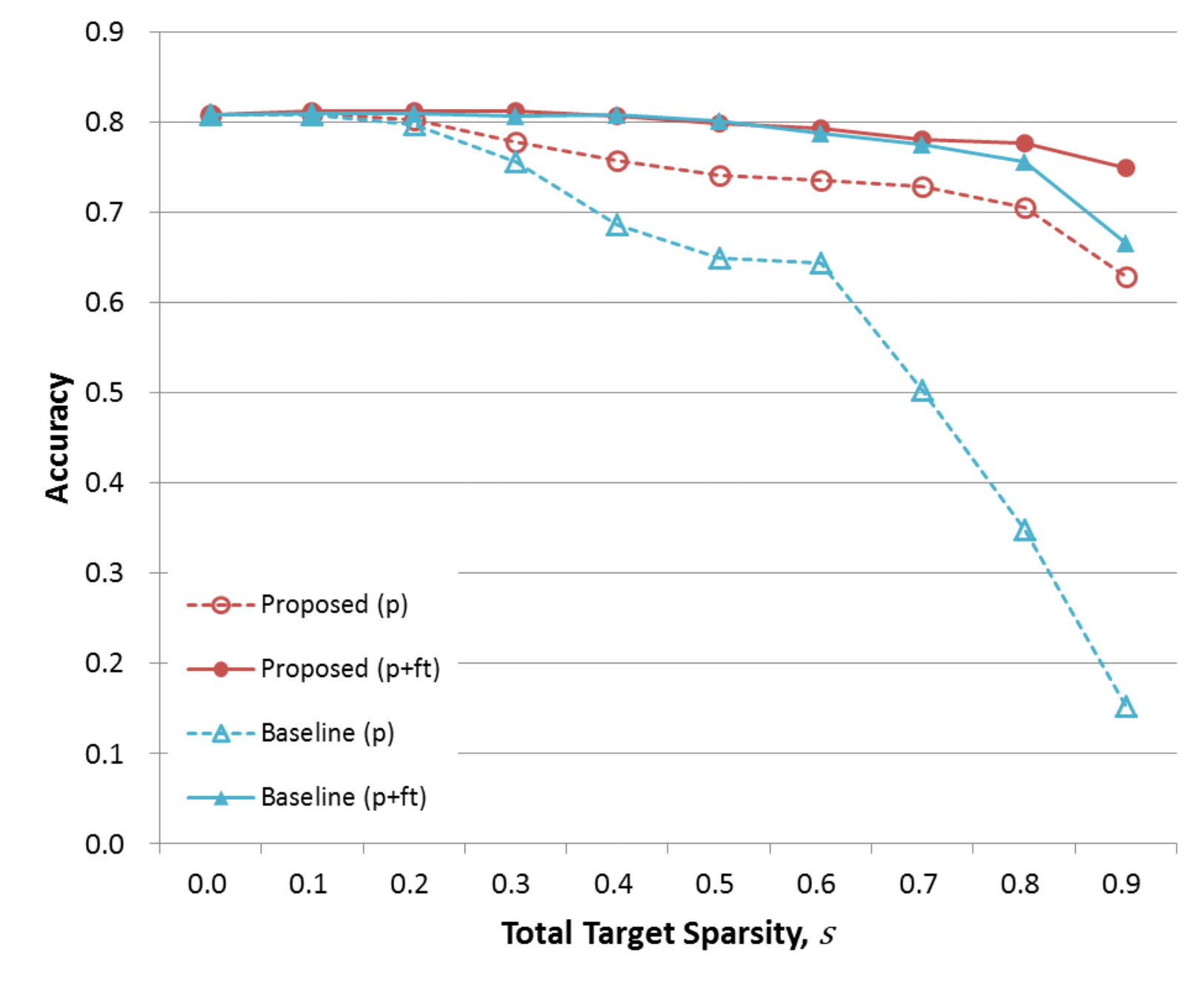}}
    \caption{Classification accuracy comparison against the baseline and the proposed method using a simple DNN model on the CIFAR-10 dataset. For compression, weight pruning \cite{han2015learning} is used. (p) and (p + ft) mean pruning only and fine-tuning after pruning, respectively.}
    \label{fig:cifar10simple_wemag}
\end{figure}

\begin{figure}[t]
    \centerline{\includegraphics[width=0.85\columnwidth]{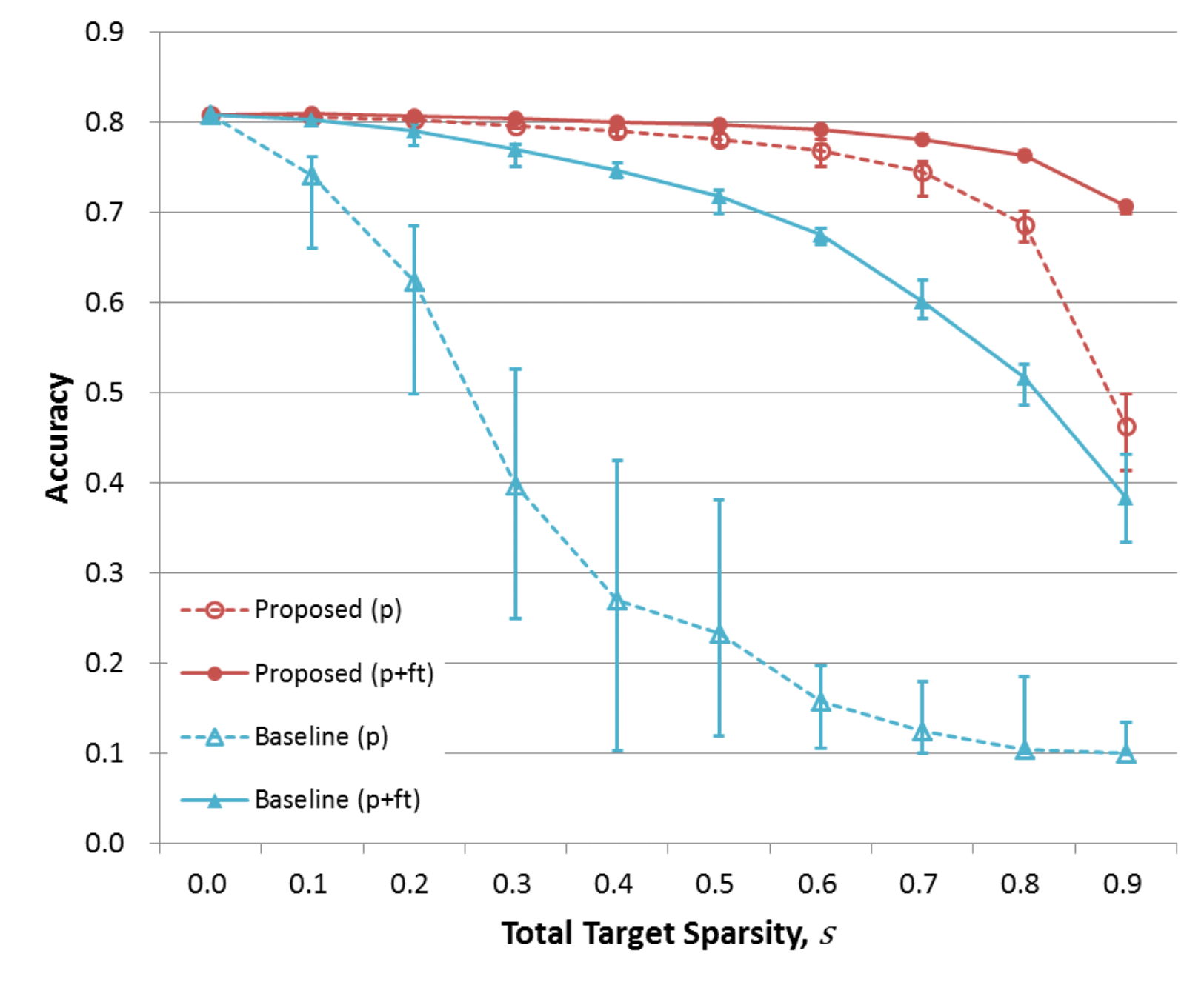}}
    \caption{Classification accuracy comparison against the baseline and the proposed method using simple DNN model on the CIFAR-10 dataset. For compression, random channel pruning is used. We plot the median value for 10 trials and the vertical bar at each point represents the max and min value. (p) and (p + ft) mean pruning only and fine-tuning after pruning, respectively.}
    \label{fig:cifar10simple_chrand}
\end{figure}

\begin{figure}[t]
    \centerline{\includegraphics[width=0.85\columnwidth]{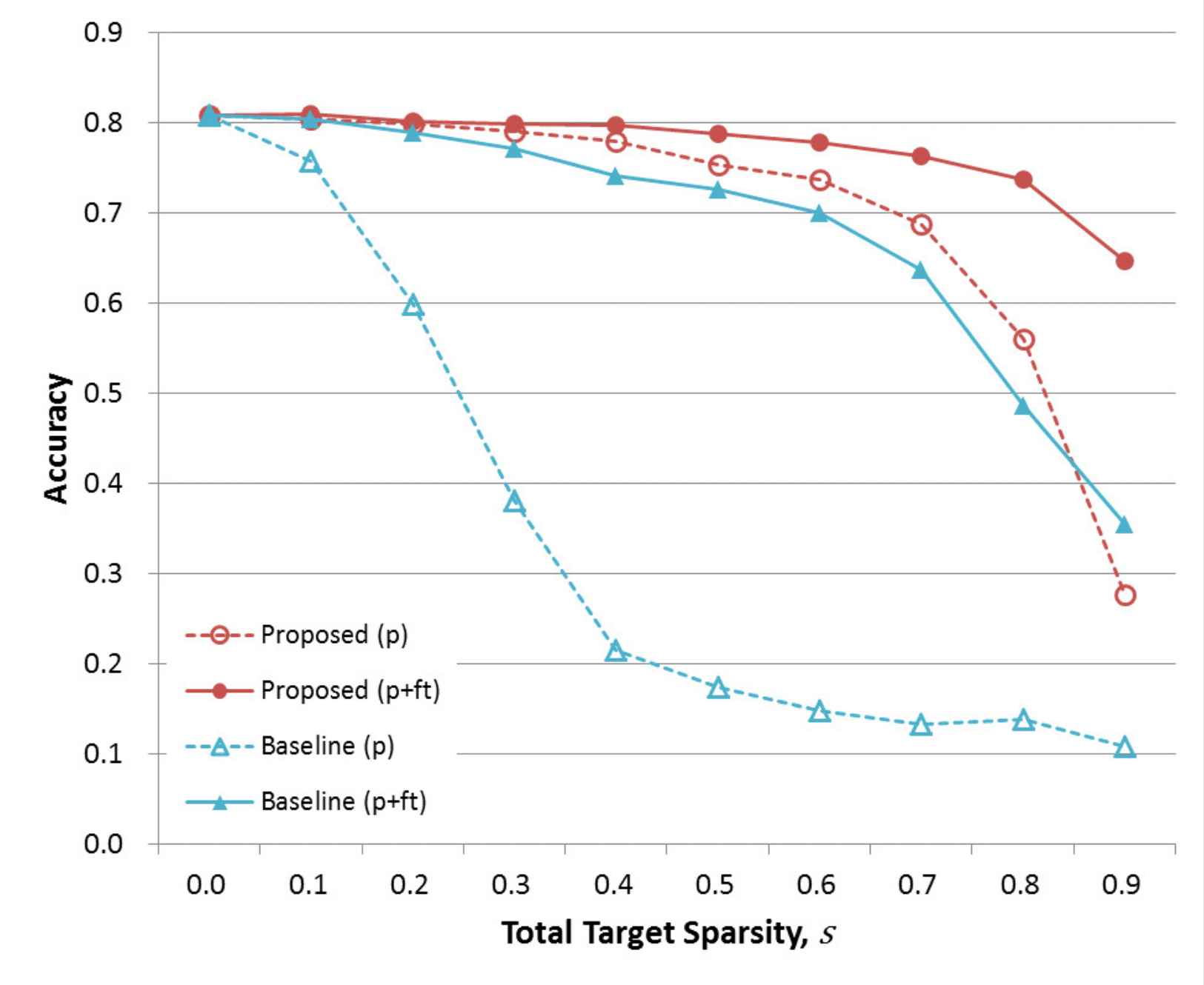}}
    \caption{Classification accuracy comparison against the baseline and the proposed method using simple DNN model on the CIFAR-10 dataset. For compression, channel pruning \cite{li2016pruning} is used. (p) and (p + ft) mean pruning only and fine-tuning after pruning, respectively.}
    \label{fig:cifar10simple_chmag}
\end{figure}

\section{EXPERIMENTAL RESULTS}
In this section, we investigate how much our proposed scheme affects the performance of a pruned model.
To do this, we carried out an image classification task with the simple DNN and VGG-16 \cite{simonyan2014very} model on the CIFAR-10 \cite{krizhevsky2014cifar} and ImageNet \cite{Deng09imagenet:a} dataset.  

In all experiments, we compared our layer-wise sparsity method with the uniform sparsity method which we use as the baseline.
To investigate how much robust our proposed scheme is to different pruning methods, we applied three pruning methods (magnitude-based weight pruning by \npcite{han2015learning}, random channel pruning, and magnitude-based channel pruning by \npcite{li2016pruning}) to DNN architectures. 

We set $\xi_l$ as $3\times w_l\times h_l\times c_{l-1}$ where $c_{l-1}$ is the number of input channels for the $l$-th layer and $w_l$ and $h_l$ is the spatial filter size of the $l$-th layer.
In other words, we wanted to remain at least 3 channels for performance.    

We implemented the proposed method using Keras \cite{chollet2015keras}.
For the VGG-16 model on the ImageNet dataset, we used pre-trained weights in Keras, and for the VGG-16 model on the CIFAR-10 dataset, we used pre-trained weights from \cite{cifar-vgg}.
The simple DNN model was designed and trained by ourselves. 

\subsection{Simple DNN Model on CIFAR-10 Dataset}
Table \ref{tbl:keras-model} shows the architecture of the simple DNN model used in this experiment.  
To compress the model, we applied pruning to Conv2-4 and FC1 layers. 
After pruning is done, we can additionally apply (optional) fine-tuning to improve performance.
Therefore we checked the accuracy of both pruned and fine-tuned models to investigate the resilience of the proposed method. 
For fine-tuning, we ran 3 epochs with learning rate = 0.0001.

Because the computation of layer-wise sparsity is not affected by the choice of pruning method, we need to compute layer-wise sparsity for the model only once.
In the subsections below, therefore, we shared the computed layer-wise sparsity with all three pruning methods.

\subsubsection{Weight Pruning}
To prune weights, we used magnitude-based pruning \cite{han2015learning}. 
In other words, we pruned weights that have small absolute values because we can consider that they do not contribute much to the output.

Figure \ref{fig:cifar10simple_wemag} shows the performance after compression. 
As we can see in the figure, the proposed method outperforms under all total target sparsities in pruning only case, and achieves similar or better results after fine-tuning. 
At large total target sparsity, the effect of the proposed method becomes apparent. 
For example, at total target sparsity 0.9, the accuracy of the proposed method drops by 0.179 after pruning only, while the baseline drops by 0.656. 

\subsubsection{Random Channel Pruning}
To eliminate the effect of the details of a pruning algorithm, we also applied random pruning in channel. 
Given the total target sparsity, we computed the layer-wise sparsities and the number of required parameters to be pruned in order. 
Then the number of pruned channels is determined by applying floor operation to the number of pruned parameters.
According to the computed number of pruned channels, we \textit{randomly} selected which channels are pruned and repeated the selection 10 times.

We compared the classification accuracy of the proposed layer-wise sparsity scheme against the baseline. 
Figure \ref{fig:cifar10simple_chrand} shows the results. 
Similar with weight pruning, the proposed layer-wise sparsity scheme also outperforms the baseline under all total target sparsities. 
As we can see in figure \ref{fig:cifar10simple_chrand}, the proposed method conducts compression reliably compared with the baseline method (smaller height of min-max bar than the baseline).
Interestingly, the result of the proposed method after pruning outperforms the baseline method after fine-tuning.

\begin{figure*}
    \centerline{\includegraphics[width=1.8\columnwidth]{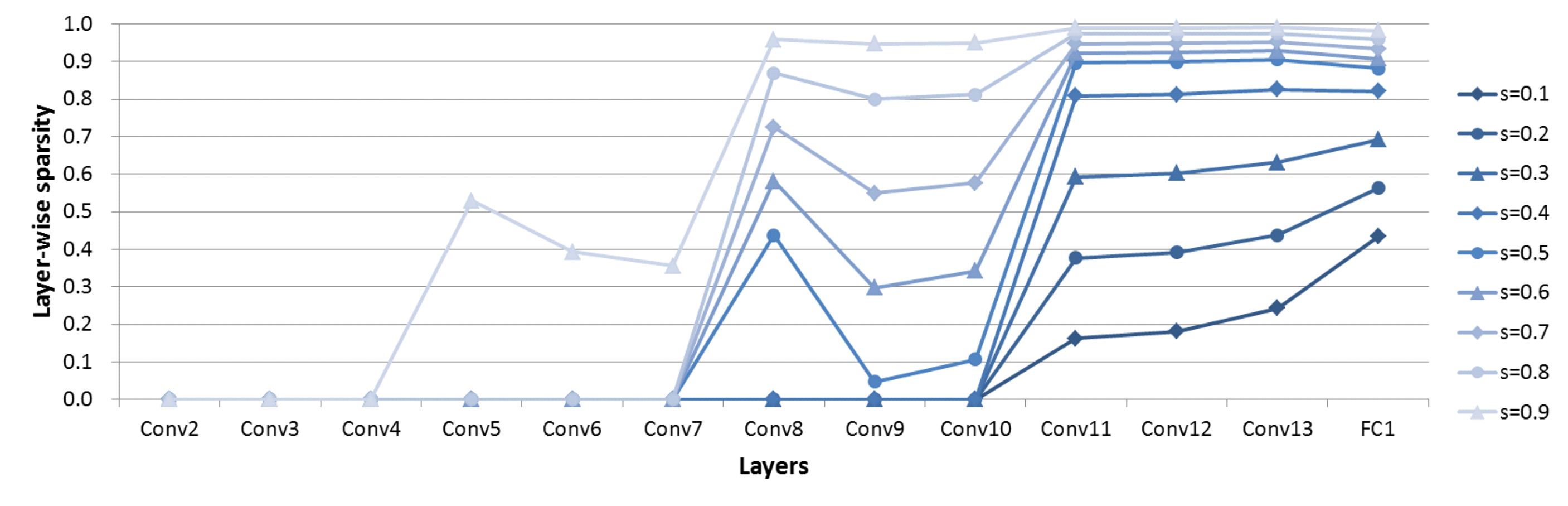}}
    \caption{The computed layer-wise sparsity of VGG-16 model on the CIFAR-10 dataset, given the total target sparsity $s$.}
    \label{fig:cifar10vgg_sparsity}
\end{figure*}

\begin{figure}[t]
    \centerline{\includegraphics[width=0.9\columnwidth]{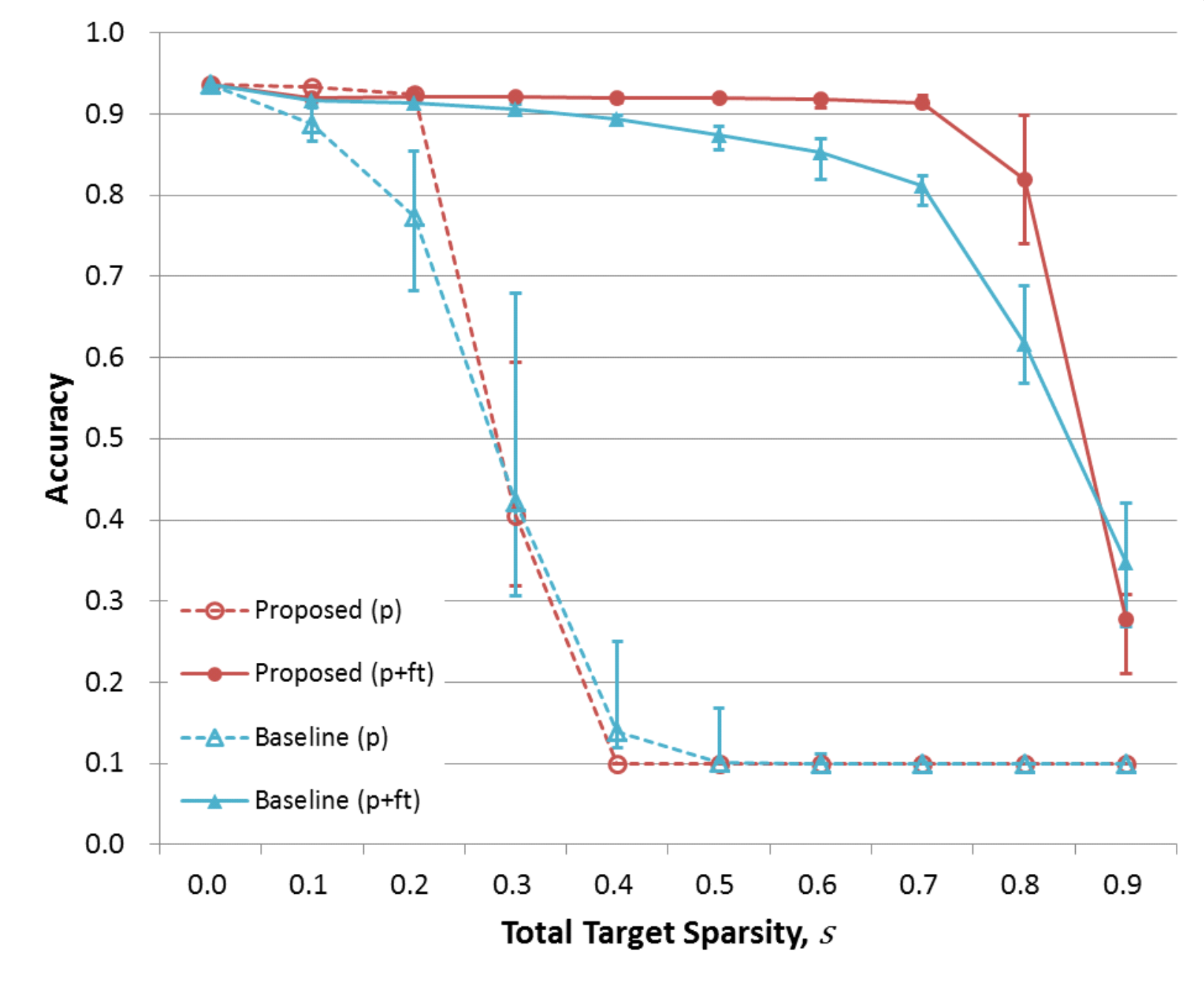}}
    \caption{Classification accuracy comparison against the baseline and the proposed methods using the VGG-16 model on the CIFAR-10 dataset. For compression, random channel pruning is used. We plot the median value for 10 trials and the vertical bar at each point represents the max and min value. (p) and (p + ft) mean pruning only and fine-tuning after pruning, respectively.}
    \label{fig:cifar10vgg_chrand}
\end{figure}

\begin{figure}[t]
    \centerline{\includegraphics[width=0.9\columnwidth]{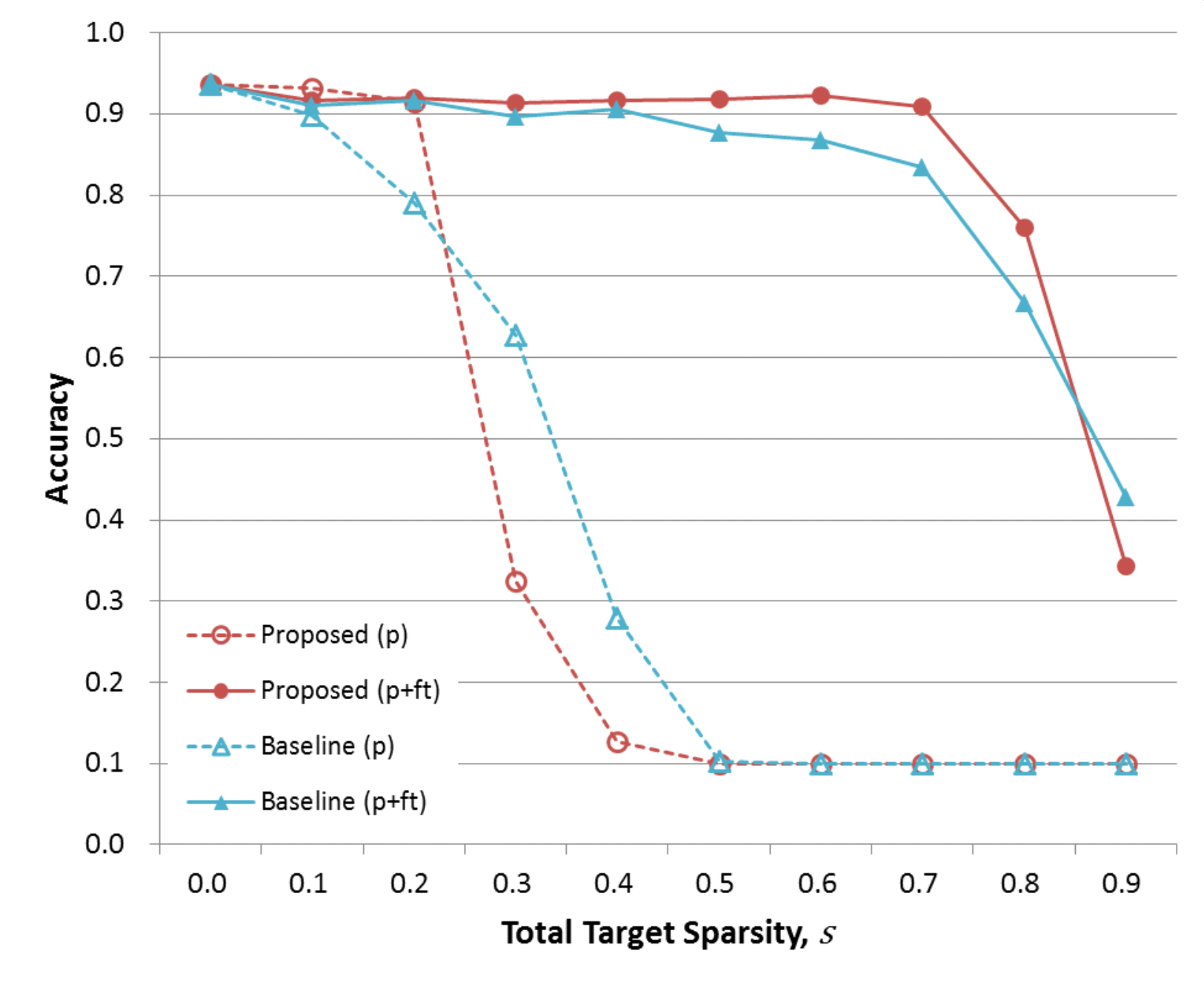}}
    \caption{Classification accuracy comparison with the baseline and the proposed method of VGG-16 model on the CIFAR-10 dataset.  For compression, channel pruning \cite{li2016pruning} is used. (p) and (p + ft) mean pruning only and fine-tuning after pruning, respectively.}
    \label{fig:cifar10vgg_chmag}
\end{figure}

\subsubsection{Channel Pruning}
We used channel pruning \cite{li2016pruning} to compress the model. 
The authors used the sum of absolute weights in a channel as the criteria for pruning. 
In other words, channels that have small magnitude weights are pruned. 
Figure \ref{fig:cifar10simple_chmag} shows the results. 

Numerically, the proposed method achieved up to 58.9\% better classification accuracy than the baseline using the same total target sparsity.
Please refer to the accuracy of the Baseline (p) and Proposed (p) at total target sparsity 0.6 in figure \ref{fig:cifar10simple_chmag}.
In terms of compression ratio, the proposed method can prune up to 5 times more parameters than the baseline while retaining the same accuracy.
Please compare the accuracy of the Baseline (p) at total target sparsity 0.1 and Proposed (p) at total target sparsity 0.5 in figure \ref{fig:cifar10simple_chmag}.

\subsection{VGG-16 on CIFAR-10 Dataset}
In this experiment, we applied pruning to all Conv and FC layers except the first and the last layers (Conv1 and FC2) in the VGG-16 model.
For fine-tuning, we ran 3 epochs with learning rate = 0.0001 in all pruning cases.

Figure \ref{fig:cifar10vgg_sparsity} shows the computed layer-wise sparsities under different total target sparsities, $s$. 
The computed sparsities confirm our assumption that all layers have different importance for the given task.



\subsubsection{Random Channel Pruning}
Similar to the simple DNN model, we repeated the random channel pruning 10 times and the result of pruning is shown in figure \ref{fig:cifar10vgg_chrand}. 
Surprisingly, the proposed method achieves almost 90\% maximum accuracy when total target sparsity is 0.8.
Although the VGG-16 model is considered already quite overparameterized, we can still say that the proposed method efficiently compresses DNN models.
The proposed method conduct compression reliably compared with the baseline method (smaller height of min-max bar than the baseline) as we can see in figure \ref{fig:cifar10vgg_chrand}.
Though the performance after pruning is the same with the random guessing or worse than the baseline when $s > 0.2$, the performance is almost recovered (except the case $s=0.9$) after fine-tuning.
This demonstrates that considering layer-wise sparsity helps not only pruning but also performance improvement with fine-tuning.

\begin{table*}[t]
 \caption{Performance evaluation for the baseline, other state-of-the-art methods and the proposed method using a VGG-16 model on the CIFAR-10 dataset. 
 We pruned Conv layers only for fair comparison.
 ratio(param) and ratio(FLOP) mean pruned ratio in number of parameters and FLOPs, respectively.
 To compute $\Delta_{err}$, ratio(param) and ratio(FLOP) of other methods, we referred the test error, \# params and \# FLOPs written in their papers.}
  \label{tbl:cifarvgg_results}
  \setlength{\tabcolsep}{9.5pt}
  \centering
  \begin{tabular}{c|c|c|c|c|c|c}
     \hline
     Method  & Error (\%)  & $\Delta_{err}$ & \# params. & ratio(param) & \# FLOPs & ratio(FLOP) \\ \hline
     Original & 6.41 & & 15.0M & & ${3.13 \times 10^8}$ & \\ \hline
     Baseline & 7.42 & +1.01 & 5.4M & 64.0\% & ${1.13 \times 10^8}$ & 64.0\% \\
     \cite{li2016pruning} & 6.60 & -0.15 & 5.4M & 64.0\% & ${2.06 \times 10^8}$ & 34.2\%  \\
     \textbf{Proposed} & \textbf{6.25} & \textbf{-0.16} & 5.4M & 64.0\% & ${2.46 \times 10^8}$ & 21.6\% \\ \hline
     \cite{ayinde2018building}-A & \textbf{6.33} & +0.13 & 3.23M & 78.1\% & ${1.86 \times 10^8}$ & 40.5\%  \\
     \cite{ayinde2018building}-B & 6.70 & +0.50 & 3.23M & 78.1\% & ${1.86 \times 10^8}$ & 40.5\%  \\
     \textbf{Proposed} & 6.53 & \textbf{+0.12} & 3.23M & 78.5\% & ${2.13 \times 10^8}$ & 32.1\% \\ 
    \hline
  \end{tabular}
\end{table*}

\subsubsection{Channel Pruning}
Figure \ref{fig:cifar10vgg_chmag} shows the results. 
Though the accuracy values after pruning only are worse than the baseline, the degree of performance improvement after fine-tuning is better than the baseline when $s > 0.2$ except $s=0.9$. 

Similar to the results of random channel pruning, the proposed method maintains the performance within 3\% of the original model until the total target sparsity becomes 0.7.

Table \ref{tbl:cifarvgg_results} shows the performance comparison with other channel pruning methods. 
For fair comparison, we pruned Conv layers only.
For fine-tuning, we ran 100 epochs with a constant learning rate 0.001 and select the best accuracy. 
As we can see in the Table, the proposed method won first place when the compressed model has 5.4M parameters and is in second place when the compressed model has 3.23M parameters. 
However, in performance drop ($\Delta_{err}$), the proposed method outperforms other methods in both cases. 
Our proposed method prunes more filters from the later layers than those from the former layers.
It performs better in reducing the number of parameters, while it does not in reducing FLOPs as we can see in Table \ref{tbl:cifarvgg_results}.
However, reducing FLOPs can be easily achievable by reformulating (Eq. \ref{eqn:sparsity_condi}).



\begin{figure*}[t]
    \centerline{\includegraphics[width=1.8\columnwidth]{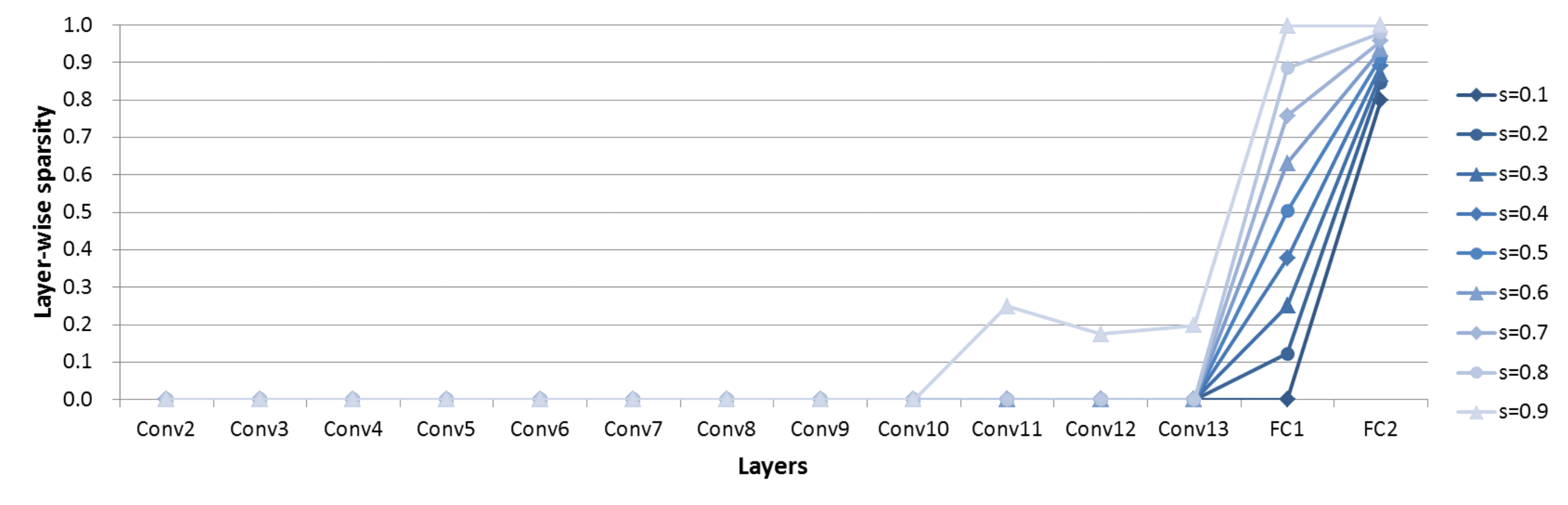}}
    \caption{The computed layer-wise sparsity of VGG-16 model on the ImageNet dataset, given the total target sparsity $s$.}
    \label{fig:imagenetvgg_sparsity}
\end{figure*}

\subsection{VGG-16 on ImageNet Dataset}
In this experiment, we applied pruning to all Conv and FC layers except the first and the last layers (Conv1 and FC3) in the VGG-16 model on the ImageNet dataset.

\begin{figure}[t]
    \centerline{\includegraphics[width=\columnwidth]{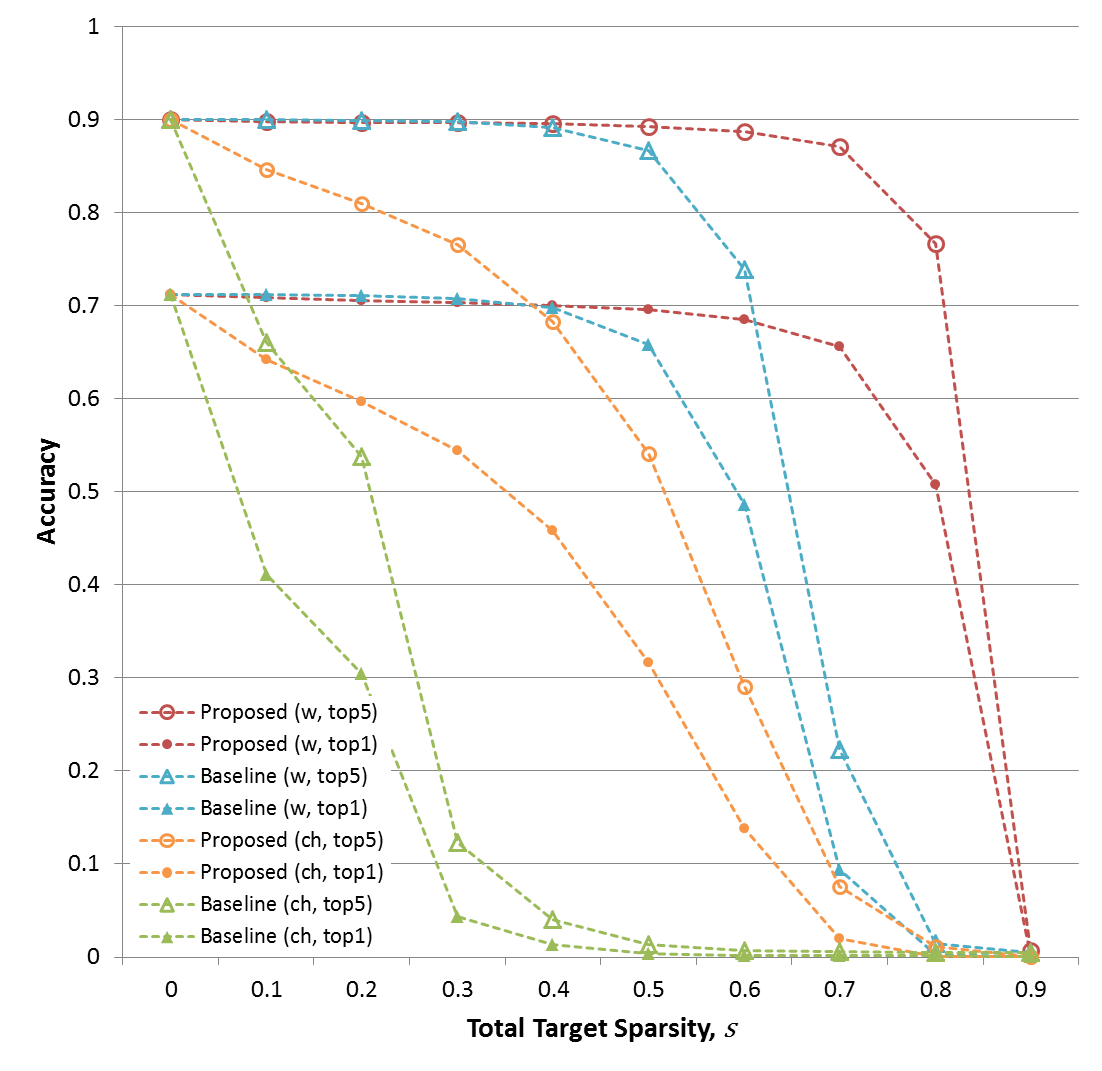}}
    \caption{Classification accuracy comparison against the baseline and the proposed method using a VGG-16 model on the ImageNet dataset.
    'w' means weight pruning and 'ch' means channel pruning.}
    \label{fig:ImageNet_vgg}
\end{figure}


Figure \ref{fig:imagenetvgg_sparsity} shows the computed layer-wise sparsities under different total target sparsities. 
As we can see in the figure, the difference of layer-wise sparsity between layers. 
Surprisingly, the figure says that we only need prune two fully connected layers (FC1 and FC2) until the total target sparsity $s$ becomes 0.8.
Such results are reasonable because more than 85\% of total parameters in VGG are concentrated in the FC1 and FC2 layers.
Therefore we can consider that there would be many redundant parameters in those layers.
However, we can also see that the proposed method computes layer-wise sparsities not considering the number of parameters only, from the figure.
For example, the number of parameters in FC1 layer is far lager than the FC2 layer but the sparsity of FC2 layer is larger than the FC1 layer. 
Conv11, Conv12, and Conv13 also represents similar results (all three layers have the same number of parameters).  

\subsubsection{Weight Pruning}
Figure \ref{fig:ImageNet_vgg} shows the performance after compression. 
The proposed method outperforms the baseline under all target sparsities both top-1 and top-5 accuracy. 
Both methods maintain the performance until $s=0.4$. 
But when $s$ becomes larger than 0.4, the proposed method shows consistently better performance.


\subsubsection{Channel Pruning}
Figure \ref{fig:ImageNet_vgg} shows the compression results using channel pruning  \cite{li2016pruning}. 
Because the channel pruning removes bunch of parameters, distributing the number of reqired parameters to be pruned to all layers according to the layer-wise sparsity is harder than weight pruning.
Therefore the performance is worse than the weight pruning but still outperforms the baseline method in both top-1 and top-5 accuracy cases.
 

From the above results, the proposed layer-wise sparsity scheme outperforms the baseline method except for few cases.
We can validate our claim that not all layers have the same importance for the given task 
and the proposed layer-wise sparsity scheme is highly effective in various DNN models for compression by pruning. 

\section{CONCLUSION}
In this paper, we proposed a new method that automatically computes the layer-wise sparsity from the layer-wise capacity for DNN model compression, especially pruning.
Our proposed method does not require additional training or evaluation steps to compute the layer-wise sparsity, which has an advantage in terms of computation time.
Experimental results validated the efficiency of the proposed layer-wise sparsity calculation in DNN model compression. Furthermore, the estimated layer-wise sparsity varied greatly across
layers, suggesting that the information can be used to find where the most compact representation resides in the deep neural network.



\bibliography{main}
\bibliographystyle{aaai}

\bigskip
\end{document}